\documentclass[english]{article}
\usepackage[T1]{fontenc}
\usepackage[latin9]{inputenc}
\usepackage{amsthm}
\usepackage{amsmath}
\usepackage{amssymb}
\usepackage{mathdots}
\PassOptionsToPackage{version=3}{mhchem}
\usepackage{mhchem}
\usepackage{esint}

\makeatletter
\theoremstyle{plain}
\newtheorem{thm}{\protect\theoremname}
  \theoremstyle{plain}
  \newtheorem{prop}[thm]{\protect\propositionname}

\makeatother

\usepackage{babel}
  \providecommand{\propositionname}{Proposition}
\providecommand{\theoremname}{Theorem}

\begin{document}

\global\long\def\rvx{\mathsf{x}}
\global\long\def\rvy{\mathsf{y}}

\global\long\def\tensord{\bigotimes_{i=1}^{d}}
\global\long\def\cco{C_{\rvy\rvx}}
\global\long\def\ncco{V_{\rvy\rvx}}
\global\long\def\cox{C_{\rvx\rvx}}
\global\long\def\coy{C_{\rvy\rvy}}
\global\long\def\rkhsx{\mathcal{H}_{k_{\mathcal{X}}}}
\global\long\def\rkhsy{\mathcal{H}_{k_{\mathcal{Y}}}}
\global\long\def\rangesgmx{\overline{Range(C_{\rvx\rvx})}}
\global\long\def\rangesgmy{\overline{Range(C_{\rvy\rvy})}}

\title{A simpler condition for consistency of a kernel independence test}

\author{Arthur Gretton}
\maketitle
\begin{abstract}
A statistical test of independence may be constructed using the Hilbert-Schmidt
Independence Criterion (HSIC) as a test statistic. The HSIC is defined
as the distance between the embedding of the joint distribution, and
the embedding of the product of the marginals, in a Reproducing Kernel
Hilbert Space (RKHS). It has previously been shown that when the kernel
used in defining the joint embedding is characteristic (that is, the
embedding of the joint distribution to the feature space is injective),
then the HSIC-based test is consistent. In particular, it is sufficient
for the product of kernels on the individual domains to be characteristic
on the joint domain. In this note, it is established via a result
of Lyons (2013) that HSIC-based independence tests are consistent
when kernels on the marginals are characteristic on their respective
domains, even when the product of kernels is not characteristic on
the joint domain.
\end{abstract}

\section{Introduction}

The Hilbert-Schmidt Independence Criterion \cite{GreBouSmoSch05}
provides a measure of dependence between random variables $X$ on
domain $\mathcal{X}$, and $Y$ on domain $\mathcal{Y}$, with joint
probability measure $P_{XY}$ on $\mathcal{X}\times\mathcal{Y}$.
This dependence measure may be used in statistical tests of dependence
\cite{GreFukTeoSonetal08,GreGyo10}. The simplest way to understand
HSIC is as the distance between an embedding of the joint distribution
and the product of the marginals, to an appropriate feature space
\cite{SmoGreSonSch07,GreBorRasSchetal12}, which is in our case a
reproducing kernel Hilbert space. The distance covariance of \cite{SzeRizBak07}
is a special case, for a particular choice of kernel \cite{SejSriGreFuk13}.
We say the feature space is characteristic when the embedding is injective,
and uniquely identifies probability measures \cite{SriGreFukLanetal10,FukGreSunSch08,SriFukLan11}.
A test based on HSIC is consistent when product of kernels on the
domains being compared is characteristic to the joint domain \cite[Theorem 3]{FukGreSunSch08}.
This is shown to be the case e.g. when Gaussian kernels are used on
each of the domains.

We propose a simpler condition: namely, that the kernels on each of
the individual domains $\mathcal{X}$ and $\mathcal{Y}$ should be
characteristic to those domains. The result is a direct consequence
of \cite[Lemma 3.8]{Lyons13}. The result is of particular interest
since it may be easier to define characteristic kernels on individual
domains than on the joint domain. For example, characteristic kernels
may be defined on the group of orthogonal matrices \cite[Section 4]{FukSriGreSch09},
and on the semigroup of vectors of non-negative reals \cite[Section 5]{FukSriGreSch09},
however a kernel jointly characteristic to both domains (i.e., to
orthogonal matrix/non-negative vector pairs) is harder to define.

\section{Results}

We begin with a result from \cite{SriFukLan11} that characteristic,
translation invariant kernels provide injective embeddings of finite
signed measures.
\begin{prop}
[Injective embeddings of finite signed measures]\label{prop:Injective-embeddings}
Let $\mathcal{X}$ be a Polish, locally compact Hausdorff space. Let
$k(x,y)$ be a $c_{0}$-kernel, i.e. a bounded kernel for which $k(x,\cdot)\in C_{0}(\mathcal{X})\quad\forall x$,
where $C_{0}(\mathcal{X})$ is the class of continuous functions on
$\mathcal{X}$ that vanish at infinity.%
\footnote{Continuous functions vanishing at infinity are members of $f\in C(\mathcal{X})$
such that for all $\varepsilon>0$ the set $\left\{ x\::\:\left|f(x)\right|\ge\varepsilon\right\} $
is compact.%
} Assume $k(x,y)=k(x-y)$, i.e. the kernel is translation invariant.
Define as $\mathcal{F}$ the RKHS induced by $k$. The following statements
are equivalent:
\begin{enumerate}
\item $k$ is characteristic
\item The embedding of a finite signed Borel measure $\mu\in\mathcal{M}_{b}(\mathcal{X})$,
defined as
\begin{equation}
\mu\mapsto\int_{\mathcal{X}}k(\cdot,x)d\mu(x),\label{eq:injectiveFiniteSignedEmbedding}
\end{equation}
 is injective.
\end{enumerate}
\end{prop}
This result may be obtained by combining \cite[Proposition 2]{SriFukLan11},
which states that an RKHS is $c_{0}$-universal iff the embedding
in \eqref{eq:injectiveFiniteSignedEmbedding} is injective, with the
result in \cite[Section 3.2]{SriFukLan11} that translation invariant
kernels are $c_{0}$-universal iff they are characteristic.

This being the case, a minor adaptation of the proof of \cite[Lemma 3.8]{Lyons13}
leads to the following result.
\begin{thm}
[Characteristic kernels and independence measures] Let $k$ and
$l$ be kernels for the respective RKHSs $\mathcal{F}$ on $\mathcal{X}$
and $\mathcal{G}$ on $\mathcal{Y}$, with respective feature maps
$\phi$ and $\psi$. Assume both $k$ and $l$ are characteristic,
translation invariant $c_{0}$-kernels, satisfying the conditions
of Proposition \ref{prop:Injective-embeddings}. Define the finite
signed measure
\[
\theta:=P_{XY}-P_{X}P_{Y}.
\]
Define the covariance operator as the embedding of this signed measure
into the tensor space%
\footnote{The tensor product is defined such that $\left(a\otimes b\right)c=\left\langle b,c\right\rangle _{\mathcal{G}}a$,
$\forall a\in\mathcal{F},\: b.c\in\mathcal{G}$.%
} $\psi(y)\otimes\phi(x)$,
\[
C_{YX}=\int_{\mathcal{X}\times\mathcal{Y}}\psi(y)\otimes\phi(x)d\theta(x,y).
\]
Then $C_{YX}=0$ iff $\theta=0$.\end{thm}
\begin{proof}
The result $\theta=0\implies C_{YX}=0$ is straightforward. We now
prove the other direction. For every $f\in\mathcal{F}$ and $B\in\sigma(\mathcal{Y})$,
we define the finite signed Borel measure
\[
\nu_{f}(B)=\int_{\mathcal{X}\times\mathcal{Y}}\left\langle \phi(x),f\right\rangle _{\mathcal{F}}\mathbb{I}_{B}(y)d\theta(x,y),
\]
where $\mathbb{I}_{B}(\cdot)$ is the indicator of the set $B.$ The
embedding of this measure to $\mathcal{G}$ is injective, and is written
\begin{align*}
\mu_{\nu_{f}} & =\int\psi(y)\left\langle \phi(x),f\right\rangle _{\mathcal{F}}d\theta(x,y)\\
 & =\int\left(\psi(y)\otimes\phi(x)\right)f\: d\theta(x,y)\\
 & =\left[\int\left(\psi(y)\otimes\phi(x)\right)d\theta(x,y)\right]f\\
 & =C_{YX}f=0,
\end{align*}
where we have used the linearity of the tensor product
\[
(a\otimes b)c=T_{c}(a\otimes b)=\left\langle b,c\right\rangle a.
\]
Since the embedding $\mu_{\nu_{f}(B)}$ is injective, we have that
$\nu_{f}=0$. Since this is true for all $f\in\mathcal{F}$, we have
that
\[
\int_{\mathcal{X}\times\mathcal{Y}}\phi(x)\mathbb{I}_{B}(y)d\theta(x,y)=0.
\]
Define the finite signed measure on $A$, $\nu_{B}(A)=\theta(A\times B)$.
The above equation can be interpreted as the embedding of this measure
to $\mathcal{F}$,
\[
\mu_{\nu_{B}}=\int_{\mathcal{X}\times\mathcal{Y}}\phi(x)\mathbb{I}_{B}(y)d\theta(x,y)=0,
\]
hence $\nu_{B}=0$, given that the embedding $\mu_{\nu_{B}}$ is injective.
We conclude that $\theta(A\times B)=0$ for all Borel sets $A,B$,
and hence $\theta=0$.
\end{proof}
An important point to note is that the embedding of $\theta$ need
not be characteristic to all probability measures: only the embeddings
of each of the individual dimensions $\mathcal{X}$ and $\mathcal{Y}$
need be characteristic. A second point is that a consistent test still
requires characteristic kernels on both domains; it is not sufficient
for one domain alone to have a characteristic kernel. A simple example
can be used to illustrate the resulting failure mode: $\mathcal{X}:=\mathbb{R}$
with a characteristic kernel, $\mathcal{Y}:=\mathbb{R}$ with the
linear kernel $l(y_{1},y_{2})=y_{1}y_{2}$, and points are distributed
uniformly on a circular ring centered at the origin. The data are
dependent, but HSIC with these kernels will not detect this dependence.

\textbf{Acknowledgements:} Thanks to Joris Mooij, Jonas Peters, Dino
Sejdinovic, and Bharath Sriperumbudur for helpful discussions.

\bibliographystyle{plain}
\bibliography{bibfile}
\end{document}